\documentclass[sigconf,nonacm]{acmart}
\usepackage{algorithm}
\usepackage{algpseudocode}
\usepackage{amsmath}
\usepackage{booktabs}
\usepackage{enumitem}
\usepackage{mathtools}
\usepackage{xspace}
\usepackage{appendix}
\usepackage{multirow}

\setcopyright{none}
\renewcommand\footnotetextcopyrightpermission[1]{}
\newcommand{\method}{\textsc{LRMoE}\xspace}

\begin{document}

\title[\method]{\method: Demand-Side Load Re-Calibration for Efficient MoE Training on Memory-Constrained Nodes}

\author{Jiaming Hu}
\affiliation{\institution{JD.com}\city{Beijing}\country{China}}
\author{Zhida Jiang}
\authornote{Zhida Jiang and Zhen Chen are co-corresponding authors.}
\affiliation{\institution{JD.com}\city{Beijing}\country{China}}
\author{Zhaolong Xing}
\affiliation{\institution{JD.com}\city{Beijing}\country{China}}
\author{Yu Li}
\affiliation{\institution{JD.com}\city{Beijing}\country{China}}
\author{Zhen Chen}
\authornotemark[1]
\affiliation{\institution{JD.com}\city{Beijing}\country{China}}
\author{Ke Zhang}
\affiliation{\institution{JD.com}\city{Beijing}\country{China}}

\begin{abstract}
Sparse Mixture-of-Experts (MoE) models offer scalable capacity for Web applications, but uneven routing creates computation and communication bottlenecks. Existing placement and replication methods redistribute this workload, yet additional expert copies compete with token-related buffers for limited GPU memory. We present \method, a memory-aware MoE training system that jointly recalibrates backward demand and balances the retained workload. During forward execution, it plans retention quotas and expert placement under resource constraints while preserving all routed computations. During backpropagation, a load recalibration metric estimates each token-expert pair's contribution to the expert update and selects pairs within these quotas. A capacity-aware routing algorithm distributes retained pairs among eligible instances of their original experts. Physical compaction then forms smaller communication payloads and variable-size grouped GEMMs, reducing backward traffic, computation, and execution-buffer demand. Experiments on Qwen3-30B-A3B and GLM4.5-106B-A12B show that \method reduces cumulative MoE-layer training time to fixed-loss targets by 21.2-44.8\% relative to SYMI.
\end{abstract}


\maketitle

\section{Introduction}
\label{sec:introduction}

Web-scale search and recommendation services rely on learned models~\cite{huang2013dssm,covington2016youtube}. Training these models is an infrastructure concern: finite GPU memory and interconnect bandwidth constrain the workloads that a cluster can support. Sparse Mixture-of-Experts (MoE) architectures increase parameter capacity while activating only a small subset of experts per token~\cite{shazeer2017moe,gshard,switch,deepspeedmoe,deepseekmoe}. For MoE-based Web services, this shifts part of the scaling challenge from dense computation to managing routed work across devices. We study this Web-infrastructure problem at the training-system level: how to coordinate expert placement, communication, and backward workload under a fixed per-GPU resource budget, rather than improve a particular search or recommendation model.

Expert Parallelism (EP) distributes experts across GPUs and exchanges routed token--expert pairs through all-to-all communication~\cite{gshard,megablocks,deepspeedmoe}. Each pair contributes to forward and backward communication, expert computation, and saved-activation storage. Input-dependent routing can concentrate work on GPUs hosting popular experts, delaying synchronous execution~\cite{topoep}. Replication distributes this work but consumes memory and introduces state-transfer and gradient-synchronization costs. Even balanced placement can therefore leave more backward work than the available memory and communication capacity can accommodate. The resulting resource-management problem concerns both where routed pairs execute and how much backward work is retained~\cite{ultraep,topoep}.

Existing systems adapt expert placement and execution to uneven loads. FasterMoE shadows popular experts, FlexMoE adapts placement, and SmartMoE searches hybrid parallelization plans~\cite{fastermoe,flexmoe,smartmoe}. SYMI decouples model parameters from optimizer states, while LAER-MoE reconstructs expert instances from sharded parameters~\cite{symi,laermoe}. UltraEP responds to realized post-gating loads, and TopoEP incorporates communication topology into planning~\cite{ultraep,topoep}. These techniques improve utilization, but replication still consumes expert storage and incurs state-transfer and gradient-synchronization costs. For a fixed admitted token--expert pair set, placement and replication change where work executes without reducing its logical computation. Replicas also compete with token buffers for memory. Under a constrained budget, redistributing that workload can therefore leave total demand above the available service capacity.

Selective-Backprop reduces backward work by prioritizing high-loss training examples~\cite{selectivebackprop}. We instead select routed token--expert pairs by their estimated contributions to the expert update. Once an MoE layer receives its incoming gradients, these can be combined with saved activations and routing weights to characterize pairwise contributions~\cite{goodfellowperexample}. At a protected expert projection, a pair's weight-gradient contribution is the outer product of its weighted incoming gradient and saved projection input. We derive the \emph{Load Recalibration Metric (LRM)} from the squared norm of this contribution, allowing pairs to be ranked without materializing individual weight-gradient matrices. LRM is a projection-level proxy for selecting backward contributions under a limited budget. Because these gradients are unavailable during forward execution, we preserve all routed forward computations and defer contribution-aware selection until backpropagation. MegaBlocks addresses irregular expert execution with block-sparse computation~\cite{megablocks}. Our selection must likewise translate into smaller physical workloads. Removing work from an unconstrained path does not accelerate execution, and reducing total work without rerouting can leave imbalances intact. The selected backward update remains approximate, so its effect on learning must be evaluated empirically rather than inferred from pairwise scores alone.

We present \method, a memory-aware MoE training system that combines contribution-aware selection with resource-constrained quota planning and load routing. Its two-phase design separates how much backward work the system can afford from which pairs fill that budget. After forward gating reveals the routed load, the controller sets retention quotas for each expert and plans replica placement. These decisions account for full forward execution and expected backward work within memory and communication limits. The full forward pass then preserves the context needed for backpropagation. Once gradients arrive, LRM selects the highest-scoring pairs within each expert's quota. Capacity-aware routing assigns retained pairs to eligible instances, and physical compaction groups their gradients and saved activations by destination and expert. The resulting dense payloads reduce backward all-to-all traffic. Variable-size grouped GEMMs use the actual retained row counts without padding workloads across experts. An inverse mapping restores gradients to the layout expected by surrounding layers. This coordination reduces backward communication, expert computation, and execution-buffer demand, while requiring selection and packing overhead to be included in evaluation. Our contributions are:
\begin{itemize}
\item We introduce a contribution-aware calibration metric and reduce backward communication and computation through grouped-GEMM workloads.
\item We develop memory-aware routing that re-balances retained workloads across eligible expert instances.
\item On Qwen3-30B-A3B and GLM4.5-106B-A12B, \method reduces cumulative MoE-layer training time to fixed-loss targets by 21.2-44.8\% versus state-to-the-art baselines, a separate GPU replay shows a 57.8\% reduction in the subtotal of expert-placement, replica, and token-buffer memory relative to fixed ownership.
\end{itemize}

\section{Background and Motivation}
\label{sec:background}

\subsection{MoE and Load Balancing}

In Transformer models~\cite{transformer}, a Mixture-of-Experts (MoE) layer replaces a dense feed-forward block with experts and a learned router.  For each token, the router activates only a small subset of experts, increasing model capacity without executing every parameter~\cite{shazeer2017moe,gshard,switch}.  Expert parallelism (EP) partitions these experts across GPUs: it dispatches each routed token to the GPU hosting its selected expert and combines the returned expert outputs~\cite{deepspeedmoe,megablocks}.  We call one token--expert selection a \emph{routed pair}. Routing makes the workload data dependent.  Pair counts change across inputs, layers, and iterations, producing unequal expert computation and all-to-all traffic across GPUs.  Because a collective and an MoE layer finish at the pace of their slowest participant, the busiest GPU becomes a straggler while others wait~\cite{fastermoe,flexmoe,smartmoe,symi,ultraep}. Training-time routing regularizers encourage balance in expectation, but cannot ensure balance for every realized layer invocation.  Runtime balancing is therefore still required for the instantaneous load. BASE layers use balanced token assignment, whereas Expert Choice fixes expert bucket sizes~\cite{baselayers,expertchoice}. Gating Dropout reduces communication by bypassing global routing for some tokens~\cite{gatingdropout}. These approaches modify forward routing, while our selection operates on backward contributions.

\subsection{Limitations of Existing Methods}

Production MoE training must accommodate model parameters, gradients, optimizer states, and runtime buffers within a fixed GPU budget. When many experts share each GPU, persistent training states leave limited memory for additional expert instances and transient execution buffers. Offloading moves training states to host memory~\cite{zerooffload,esmoe}. Input-dependent routing further concentrates demand on selected experts, so sufficient aggregate resources do not necessarily imply that every device can serve its assigned workload. Existing systems provide spatial flexibility through expert placement, replication, and adaptive parallelization~\cite{fastermoe,flexmoe,smartmoe}. EfficientMoE adapts expert scheduling to predicted loads~\cite{efficientmoe}, while NetMoE rearranges samples to exploit routing locality~\cite{netmoe}. SYMI and LAER-MoE reduce state-management costs~\cite{symi,laermoe}, while UltraEP adapts to realized post-gating loads~\cite{ultraep}. These mechanisms can improve utilization and communication locality, but they preserve the logical expert computation required by a fixed admitted token--expert pair set. Their ability to accommodate hot experts therefore remains constrained by the capacity available for executing that work. Consider a common backward service window. Let $W_{n\ell e}$ denote expert $e$'s backward demand for layer invocation $(n,\ell)$, and let $C^{\mathrm{bwd}}_{n\ell g}$ denote GPU $g$'s guarded service capacity within that window. Under an additive model in which demand and capacity are expressed in the same calibrated work units, a necessary feasibility condition is $\sum_e W_{n\ell e} \leq \sum_g C^{\mathrm{bwd}}_{n\ell g}$. Violating this bound makes the fixed workload infeasible within the window, regardless of redistribution. Satisfying it is not sufficient, however: legal expert destinations are restricted, replica and staging capacity is bounded, and per-device memory and network-path constraints must also be satisfied. A placement optimizer may therefore fail to find a feasible assignment even when aggregate service capacity is sufficient. Aggregate overload and spatial infeasibility are distinct limitations, neither can be resolved merely by making GPU loads more uniform.

\subsection{Key Insight}

Our key insight is to treat backward demand as a contribution-aware scheduling variable rather than a fixed consequence of forward routing. Calibration replaces demand $W_{n\ell e}$ with retained demand $\widetilde{W}_{n\ell e}\leq W_{n\ell e}$, prioritizing estimated update contributions under constrained capacity. Load routing distributes retained work among available expert instances. These decisions control \emph{which backward work consumes capacity} and \emph{where it executes}. Incoming gradients, saved activations, and routing weights estimate each pair's contribution to the expert update~\cite{goodfellowperexample}. The selected update is approximate, so its convergence impact must be evaluated. Selection should target overloaded expert paths, since reducing unconstrained work need not shorten execution. Retained pairs must reach eligible instances of their original experts. Physical compaction reduces communication payloads and expert-GEMM token dimensions while preserving token--expert correspondence. Masking full-sized tensors alone need not yield these savings. Net benefits must account for scoring, packing, and routing overhead and the execution efficiency of the reduced workloads.

\section{\method Design}
\label{sec:related}

\begin{figure}
    \centering

    \includegraphics[width=\linewidth]{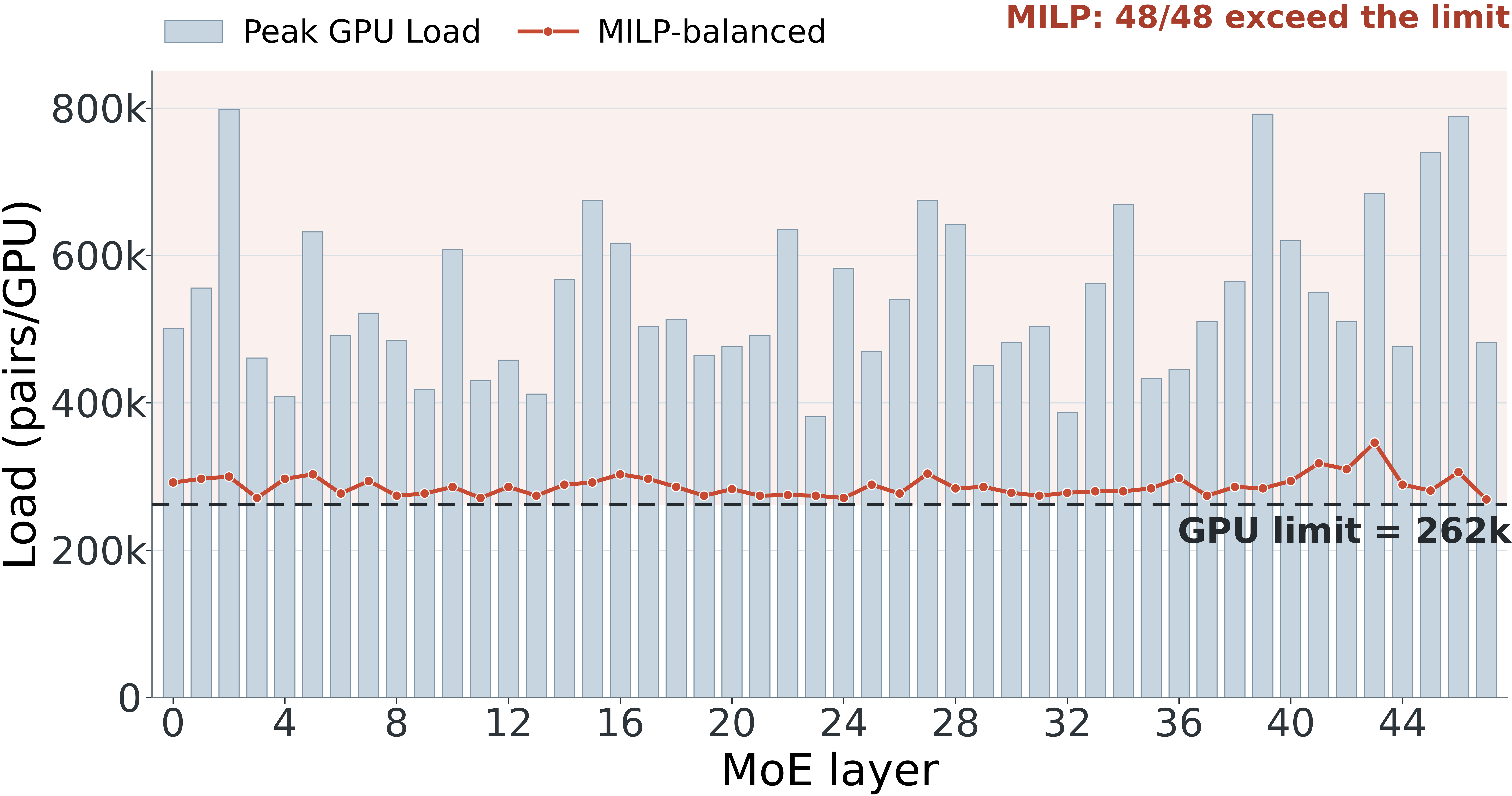}
    \caption{Uncalibrated load. All layers exceeds the maximum capacity for load balancing.}

    \label{fig:unbalanced load}
    \Description{Per-layer routed workload compared with the maximum load-balancing capacity.}
\end{figure}

\subsection{Design Overview}
\begin{figure*}
    \centering

    \includegraphics[width=1\linewidth]{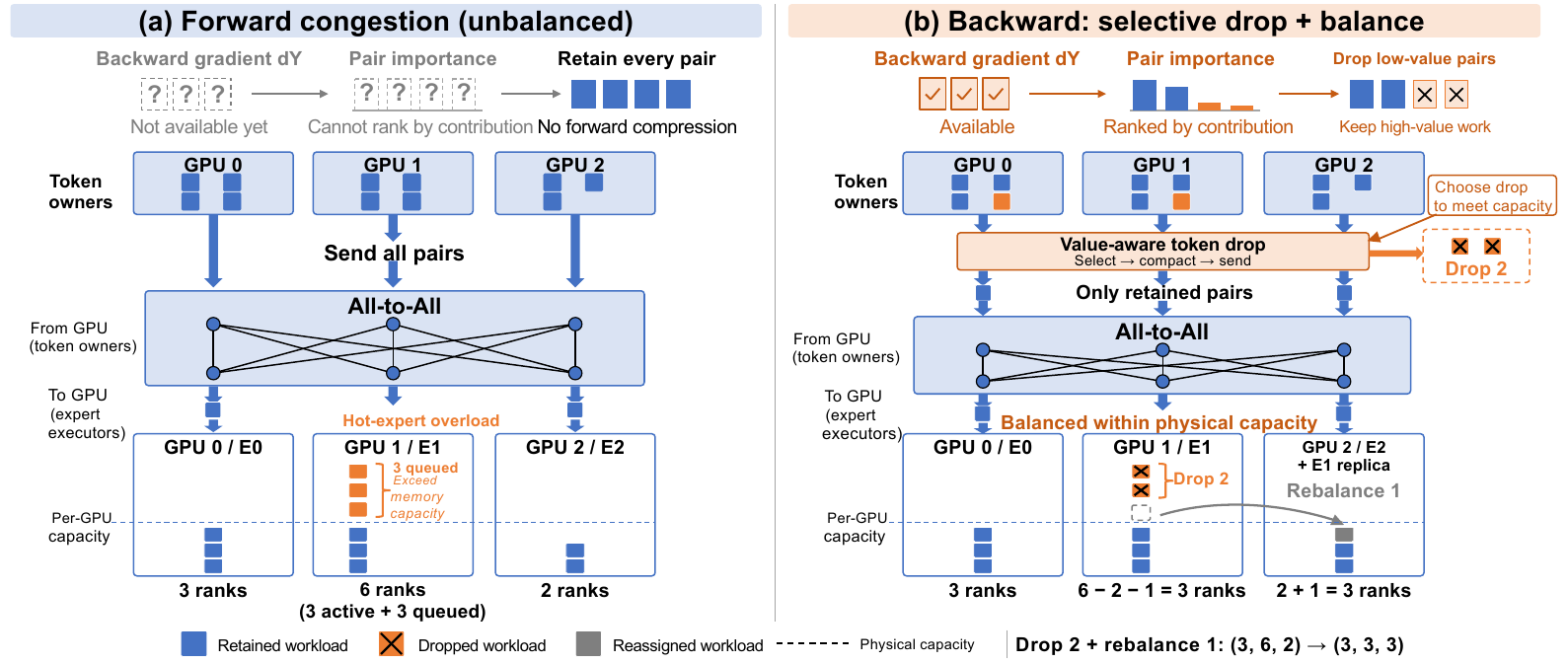}
    \caption{Selective backward load calibration for uncalibrated load (a) and calibrated load (b). In backward, \method drops low-value token-expert pairs and reassign, ensuring only retained pairs with load fitting within per-GPU capacity enter All-to-All, eliminating congestion.}
    \label{fig:LC-MoE Design}
    \Description{Comparison of uncalibrated and calibrated backward workloads. Low-value token--expert contributions are removed, and retained pairs are reassigned to fit per-GPU capacity.}
\end{figure*}

Figure~\ref{fig:LC-MoE Design} presents the design of \method, which combines layer-wise load calibration, replica-aware load routing, and compact grouped-GEMM execution. The workflow separates forward-time resource planning from backward-time contribution selection. The model router first determines each token's logical experts and routing weights, while the placement plan identifies their physical execution instances. After resource admission, the complete forward pass dispatches every routed token--expert pair, executes the expert computations, and combines their outputs into the dense representation passed to the next layer. The runtime preserves the activations, required expert intermediates, routing weights, and pair identities, together with the forward executor and parameter version, establishing the correspondence needed for subsequent backward execution.

Before expert forward begins, the controller observes the routed demand $n_e$ of each expert and determines the backward retention quotas $K_e$ under the admitted placement and guarded resource capacities. Quota planning accommodates as much backward demand as the declared capacity model permits, while accounting separately for the complete forward workload, saved-state residency, replica storage, and execution buffers. With $N=\sum_e n_e$, the retention rate is $\rho=\sum_e K_e/N$, expert $e$ omits $n_e-K_e$ backward contributions. Quotas remain fixed for the invocation. At backward Combine, LRM uses incoming gradients and saved state to select the highest-valued $K_e$ pairs per expert, without another quota search.

Load routing determines how these retained contributions occupy the admitted expert instances. During forward planning, a capacity-constrained integer allocation assigns $X_{eg}$ retained pairs of expert $e$ to GPU $g$, with $\sum_g X_{eg}=K_e$ and nonzero assignments permitted only at eligible instances of that expert. The planner minimizes the peak assigned GPU load for the fixed placement rather than dividing each expert's demand equally among replicas. During backward, rerouting realizes these instance-level quotas using the selected pair identities, preserving their logical experts and assigning every retained pair exactly once. Eligibility requires the correct expert parameters and access to the matching forward state at the recorded version. Because the selected identities are unknown during planning, the admitted state-access and communication envelope must remain valid for any selected set satisfying the quotas. Backward realization therefore neither introduces new replicas nor changes the previously admitted workload.

The compact runtime translates the retention mask and routing assignments into physical communication and computation. It gathers retained gradient rows and their matching activations and intermediates in a consistent pair order, grouping them by destination and expert instance. Counts and offsets define compact All-to-All payloads and variable-size grouped GEMMs. FC2, the activation function, FC1, and backward Dispatch use the same retained identities, omitted rows are not processed in full-sized masked tensors. Returned input-gradient contributions are scatter-added at their original token and routing-slot positions in zero-initialized dense buffers. This restores the layout, not discarded contributions, router gradients follow a separate path. Replica parameter gradients are summed at logical owners before the required accumulation and distributed reductions, one optimizer update per logical expert, and weight synchronization.

\subsection{Load Recalibration Metric}
\label{sec: lc metric}

We derive a pairwise metric that quantifies the local effect of omitting an already-routed token--expert pair from the protected backward update. Recalibration keeps the logical expert assignment and routing weight fixed, and determines which contributions remain active during backward communication and computation. To compare these contributions, we first identify the parameter-gradient contribution associated with each pair, then measure the perturbation caused by removing it. This connects the selection criterion directly to the protected expert-weight gradient rather than using routing strength or incoming gradient magnitude alone. For a pair $u=(t,e)$, the chain rule expresses its contribution to the protected projection gradient as the outer product $G_u=(\alpha_u g_t)\phi_u^\top$, where $g_t$ is the incoming gradient before multiplication by the routing weight and $\phi_u$ is the saved projection input. Omitting this pair while leaving the other contributions unchanged produces the perturbation $\Delta_u=-G_u$. We measure its magnitude using the squared Frobenius norm. Because $G_u$ is rank one, this norm factorizes into the squared norms of the weighted gradient and saved input, yielding a score that combines routing strength, backward-signal energy, and expert-state magnitude without materializing the pairwise gradient matrix. The following proposition formalizes this projection-level calibration metric.
\begin{proposition}[Token-expert calibration metric]
\label{prop:Calibration-metric}
Consider a completed forward route and a protected expert projection with
weight matrix $W_e$. For a routed token-expert pair $u=(t,e)$, let
$\alpha_u$ denote its fixed routing weight, let $g_t$ denote the gradient
arriving at the routed expert output, and let $\phi_u$ denote the saved input
to the protected projection. The contribution of pair $u$ to the weight
gradient is
\begin{equation}
G_u
=
\frac{\partial \mathcal{L}_u}{\partial W_e}
=
(\alpha_u g_t)\phi_u^\top .
\label{eq:pair-gradient-contribution}
\end{equation}
If pair $u$ is omitted while all other routed pairs remain unchanged, the
protected weight gradient changes by
\begin{equation}
\Delta_u
=
-G_u
=
-(\alpha_u g_t)\phi_u^\top .
\end{equation}
We therefore define the pairwise calibration score as the squared Frobenius
norm of this perturbation:
\begin{align}
R_u
&:=
\lVert \Delta_u \rVert_F^2 \nonumber\\
&=
\left\lVert
(\alpha_u g_t)\phi_u^\top
\right\rVert_F^2 \nonumber\\
&=
\alpha_u^2
\lVert g_t\rVert_2^2
\lVert\phi_u\rVert_2^2 .
\label{eq:Calibration-metric}
\end{align}
\end{proposition}

\begin{itemize}[leftmargin=*]
    \item \textbf{Routing strength $\alpha_u^2$.}
    This term measures how strongly token $t$ was assigned to expert $e$
    during the completed forward pass. Because routing is already fixed,
    $\alpha_u$ is treated as a constant during recalibration.

    \item \textbf{Backward signal $\lVert g_t\rVert_2^2$.}
    This term measures the magnitude of the gradient propagated from the
    subsequent computation. It distinguishes routed pairs that receive strong
    optimization signals from pairs that contribute little to the current
    update.

    \item \textbf{Expert-state magnitude $\lVert\phi_u\rVert_2^2$.}
    This term measures the magnitude of the saved projection input. Two pairs
    with similar routing weights and incoming gradients may still have
    different effects on the expert update because their saved expert states
    differ.
\end{itemize}

Equation~\ref{eq:Calibration-metric} follows from the rank-one identity $\lVert ab^\top\rVert_F^2=\lVert a\rVert_2^2\lVert b\rVert_2^2$. Thus, $R_u$ equals the squared change in the protected projection's weight gradient when pair $u$ is omitted and all other contributions remain unchanged. A larger score identifies a pair whose omission induces a larger isolated gradient perturbation. The projection-level factorization and detailed derivation are provided in Appendix~\ref{app:metric-projection}. \method separates \emph{how much} backward demand can be admitted from \emph{which} contributions should be retained. During forward planning, the resource controller chooses the largest feasible retained count $K_\ell=\sum_e K_{\ell,e}$ from its tested candidates under the quota-allocation policy, HBM budget, replica-placement constraints, and communication window, then fixes the per-expert quotas $\{K_{\ell,e}\}$. Once backward gradients become available, the metric ranks pairs within each expert, and the $K_{\ell,e}$ highest-scoring pairs fill its quota. For fixed per-expert quotas and equal row costs, this selection minimizes the additive sum of isolated omission perturbations, as established in Appendix~\ref{app:metric-selection}. The selected gradient rows and their matching saved activations are physically compacted before backward communication and grouped GEMMs, turning contribution selection into a smaller workload for load routing and execution. Merely masking full-sized tensors would not reduce their communication payloads or matrix dimensions. The selection guarantee concerns the additive local calibration objective, rather than the exact aggregate gradient error or subsequent training behavior, these distinctions and the applicable assumptions are detailed in Appendices~\ref{app:metric-aggregate} and~\ref{app:metric-scope}.

\subsection{Layerwise Load Calibration}
\label{sec:llc}

\begin{algorithm}[t]
\caption{Resource-Constrained Layer-wise Load Calibration}
\label{alg:cv_Calibration}
\begin{algorithmic}[1]
\Require Routed pairs $\mathcal P_\ell=\bigcup_e\mathcal P_{\ell,e}$,
forward placement $\Pi_\ell^F$, candidate rates $\mathcal R_\ell$,
available HBM $\{M_r^{\mathrm{avail}}\}$,
bandwidth $B_{\mathrm{eff}}$, window $T_\ell^{\mathrm{ovlp}}$
\Ensure Retained pairs $\mathcal S_\ell^*$ and backward plan $\Pi_\ell^B$
\Statex \textbf{Forward-time admission}
\State $N_{\ell,e}\gets|\mathcal P_{\ell,e}|$ for each expert $e$
\For{$\rho\in\mathcal R_\ell$ in decreasing order}
    \State $(\{K_{\ell,e}\},\Pi,\widehat{\mathbf M},
    \widehat T_{\mathrm{comm}})
    \gets\Call{PlanQuotasAndReplicas}
    {\{N_{\ell,e}\},\rho,\Pi_\ell^F,B_{\mathrm{eff}}}$
    \If{$\Pi$ is feasible, $\widehat M_r\le M_r^{\mathrm{avail}}$
    for all $r$, and $\widehat T_{\mathrm{comm}}\le T_\ell^{\mathrm{ovlp}}$}
        \State $(\rho_\ell^*,\{K_{\ell,e}^*\},\Pi_\ell^A)
        \gets(\rho,\{K_{\ell,e}\},\Pi)$
        \State \textbf{break}
    \EndIf
\EndFor
\If{$\Pi_\ell^A$ is undefined}
    \State \Return \textsc{Infeasible}
\EndIf
\State $\mathcal C_\ell
\gets\Call{FullForwardAndSaveContext}{\mathcal P_\ell,\Pi_\ell^A}$

\Statex \textbf{Backward-time calibration}
\State $v_p\gets
\alpha_p^2\lVert g_t\rVert_2^2\lVert\phi_p\rVert_2^2$
for each $p=(t,e)\in\mathcal P_\ell$
\State $\mathcal S_{\ell,e}^*
\gets\Call{TopK}{\mathcal P_{\ell,e},\{v_p\},K_{\ell,e}^*}$
for each expert $e$
\State $(\Pi_\ell^B,\{\mathcal S_{\ell,e,r}^*\})
\gets\Call{PruneAndRoute}
{\bigcup_e\mathcal S_{\ell,e}^*,\Pi_\ell^A,\mathcal C_\ell}$
\State $\{(G_{\ell,e,r},\Phi_{\ell,e,r})\}
\gets\Call{CompactAndAllToAll}
{\{\mathcal S_{\ell,e,r}^*\},\mathcal C_\ell}$
\State $(\Delta W_{\ell,e,r},D_{\ell,e,r})
\gets\Call{PackedExpertBackward}
{G_{\ell,e,r},\Phi_{\ell,e,r}}$
for each active $(e,r)$
\State $\mathcal Q_{\mathrm{CPU}}
\gets\mathcal Q_{\mathrm{CPU}}\cup
\{(e,\sum_r\Delta W_{\ell,e,r})\}_e$
\State \Call{ScatterWithZeros}
${\{D_{\ell,e,r}\},\{\mathcal S_{\ell,e,r}^*\}}$
\State \Return $(\bigcup_e\mathcal S_{\ell,e}^*,\rho_\ell^*,\Pi_\ell^B)$
\end{algorithmic}
\end{algorithm}

\begin{figure}[t]
    \centering

    \includegraphics[width=\linewidth]{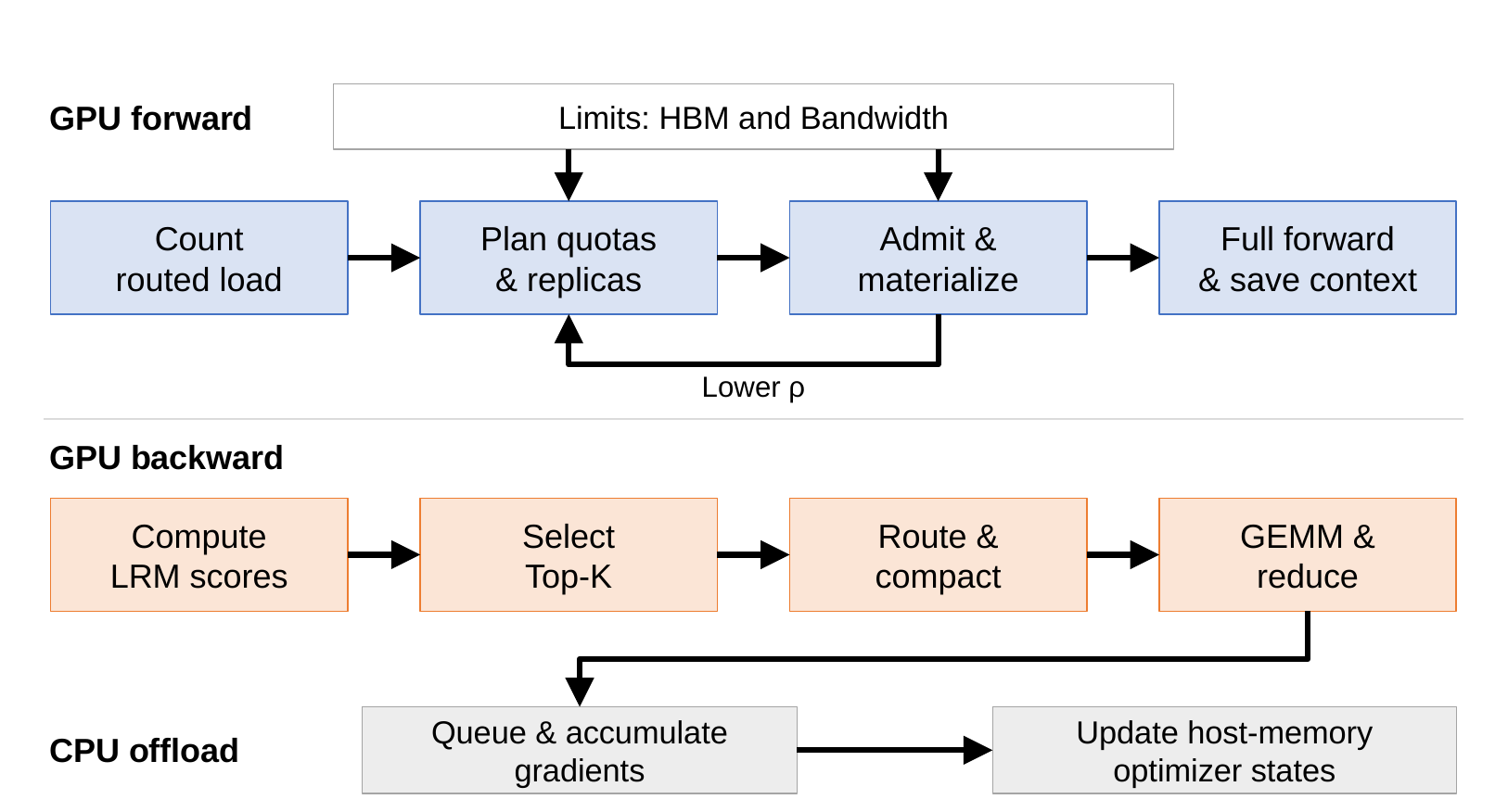}
    \caption{Two-phase Calibration workflow. Forward plans expert
    placement and retention quotas under resource constraints.
    Backward uses LRM scores to select retained token-expert pairs
    and refines the quotas to satisfy load and resource budgets.}

    \label{fig: Calibration diagram}
    \Description{Workflow with full GPU forward execution, quota planning, contribution-based backward selection, compact computation, and CPU offload.}
\end{figure}

Figure~\ref{fig: Calibration diagram} details the quota admission and execution procedure in Algorithm~\ref{alg:cv_Calibration}. For layer $\ell$, the controller counts the full routed load $N_{\ell,e}$ of each expert and allocates a candidate retained count $K_{\ell}$ into quotas satisfying $0\leq K_{\ell,e}\leq N_{\ell,e}$ and $\sum_e K_{\ell,e}=K_{\ell}$. It jointly evaluates these quotas and a proposed replica placement against the load-balancing target, GPU-memory budget, and communication window. Candidates are tested in decreasing order, starting with full retention. The first feasible candidate fixes the quotas and placement $\Pi_{\ell}^{F}$, feasibility includes the complete forward workload and its required saved state, not only the reduced backward demand. The runtime then materializes the admitted replicas, executes the full forward pass, and records pair identities, parameter versions, and backward context. If no candidate is feasible, admission fails.

These admitted quotas remain fixed when backward begins. Using the incoming gradients, saved activations, and routing weights, the controller computes the LRM score $v_p$ and selects the $K_{\ell,e}$ highest-scoring pairs within each logical expert. Per-expert selection preserves the planned load distribution, whereas an unrestricted layer-wide ranking could concentrate retained work on overloaded experts. Thus, forward planning determines the admitted counts, and backward LRM selection determines which pair identities fill them. The selected pairs then become the input to capacity-aware rerouting. Each pair is assigned exactly once to an eligible instance of its original expert within $\Pi_{\ell}^{F}$, using the retained workload and the hosting GPUs' available capacity. If $m_{\ell,e,i}$ denotes the count assigned to eligible instance $i$, routing preserves $\sum_i m_{\ell,e,i}=K_{\ell,e}$. A reassignment changes only the physical executor: the destination must provide the recorded parameter version and access to the matching saved context, with any required state movement included in the resource accounting. Rerouting therefore redistributes the selected workload without changing logical expert assignments or creating additional replicas. The resulting per-instance counts directly determine compact communication and computation. Retained gradient rows and their corresponding saved activations are packed in a consistent pair order, with counts and prefix offsets defining destination-specific payloads and expert-major tensor segments. Variable-size grouped GEMMs consume these segments: $m_{\ell,e,i}$ determines the row dimension of input-gradient products and the reduction dimension of weight-gradient products, while parameter dimensions remain unchanged. The inverse mapping restores input-gradient contributions to their original token positions, and parameter-gradient contributions are summed across instances of each logical expert before the optimizer update. In this way, the admitted quotas, selected identities, and routing assignments jointly determine the workload physically executed.

\subsection{Repacking GEMM with Load Calibration}

\begin{figure}
    \centering

    \includegraphics[width=1\linewidth]{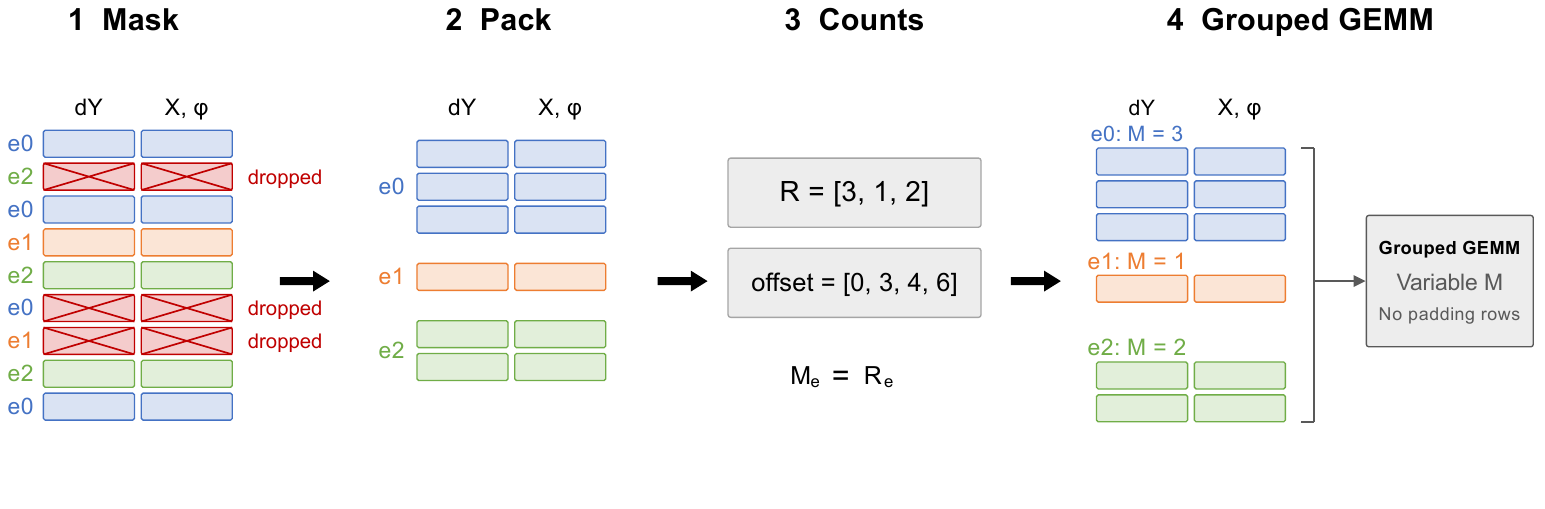}
    \caption{Expert replication reduction process. Dropped token-expert backward contributions are removed before expert-backward computation, while retained gradients and saved expert states are packed into contiguous expert segments. Their counts and offsets define variable-$M$ grouped GEMMs without padding. The reduced per-expert row counts lower both expert-backward computation and the number of physical expert replicas required to satisfy the capacity bound.}
    \label{fig:GEMM}
    \Description{Backward token--expert contributions are selected and packed into contiguous expert segments for variable-size grouped matrix multiplication.}
\end{figure}

Figure~\ref{fig:GEMM} shows how selective dropping is converted into an actual
reduction in expert computation. After identifying the retained token-expert
pairs, we remove the dropped rows and repack the corresponding gradients and
saved expert states into contiguous, expert-specific buffers. The retained
row counts and offsets are then passed to a variable-size grouped GEMM. Each
expert therefore processes only its retained rows, without computing over
zero-filled rows or padding the input back to its original size.

This distinction is essential. Zero-masking the dropped rows while preserving
the original tensor shape leaves the GEMM size and its memory requirement
largely unchanged. Repacking instead reduces the row dimension of all four
expert-backward GEMMs: the FC1 and FC2 data-gradient and weight-gradient
operations. The resulting gradients are equivalent to those produced by the
zero-masked implementation for the same retained pairs, apart from ordinary
floating-point differences. After the expert computation, the retained
results are scattered back to their original positions, and zeros are restored
only before the next dense-layer boundary.

Smaller packed GEMMs also reduce the physical capacity required for overloaded
experts. Expert replication is provisioned to accommodate the peak number of
rows assigned to each expert. Once selective dropping reduces this workload,
fewer physical instances are needed to keep the per-replica load within the
GPU capacity. In Figure~\ref{fig:GEMM}, the original expert loads are reduced
from $[4,2,3]$ to $[3,1,2]$. An expert whose retained workload now fits on one
GPU no longer requires an additional replica. The set of logical experts
remains unchanged, our method reduces only the redundant physical replicas,
along with their parameter memory, temporary buffers, and communication.

\section{\method Evaluation}

\subsection{Experimental Setup}

We evaluate MoE-layer execution on eight NVIDIA Hopper GPUs with EP8, PyTorch~2.11.0, CUDA~13.0, and BF16 precision. We use Qwen3-30B-A3B and GLM4.5-106B-A12B with a locally preprocessed Alpaca corpus, with text records concatenated and cyclically reused to construct long inputs. The sequence length is 8,192 tokens and the per-GPU microbatch size is four, yielding 32,768 source tokens per GPU. We conduct identical controlled tensors across compared methods, using 128 experts, top-8 routing, hidden dimension $H=2048$, and expert intermediate dimension $F=768$. We compare the Megatron-LM fixed-ownership strategy, native UltraEP, an independent implementation of SYMI, and \method under the same 80\,GiB per-GPU memory budget. We execute complete forward and backward passes, with 50 warm-up iterations followed by 1000 measured iterations. 

\subsection{Convergence Analysis}

\begin{figure}
    \centering

    \includegraphics[width=1\linewidth]{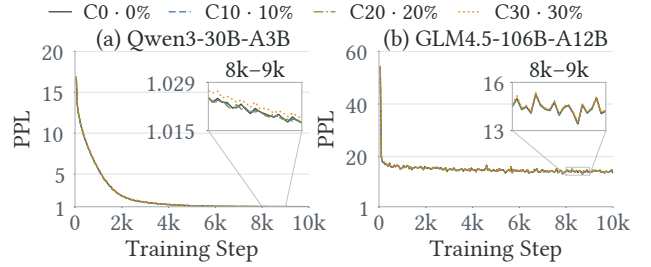}
    \caption{PPL comparison for 10000 steps training under different recalibration rate. Both models trained with Alpaca SFT dataset, EP=8, sequence length = 8192, global batch = 16. Ck denotes omitting k\% of token--expert backward contributions before backward Combine All-to-All.}

    \label{fig:ppl 10k steps}
    \Description{Perplexity curves over 10,000 training steps for two models at different backward recalibration rates.}
\end{figure}

Figure~\ref{fig:ppl 10k steps} compares training perplexity under the uncompressed baseline (C0) and three LRM-guided recalibration settings: C10, C20, and C30, corresponding to nominal reductions of 10\%, 20\%, and 30\% in the backward token-expert workload. Both models exhibit similar overall convergence trends. The late-training region reveals the differences between recalibration settings more clearly. For Qwen3-30B-A3B, the mean PPL over the final 1,000 steps is 1.015190, 1.015156, and 1.015141 for C0, C10, and C20, respectively. Meanwhile, C30 reaches 1.016270, approximately 0.106\% above the C0 baseline. These observations confirm that recalibration up to C20 closely preserves the original training convergence, whereas C30 introduces a negligible late-training gap on Qwen. This minor gap raises a complementary systems question regarding whether the substantial savings in MoE-layer execution time outweigh the cost of any additional optimization steps required to reach the exact same loss target. We address this question directly through the rigorous loss-matched comparison below.

\subsection{Convergence Time Analysis to Target Loss}

\begin{table}[t]
  \centering
  \caption{Measured cumulative MoE-layer training time to each target loss (minutes / step), including complete forward and backward execution.
   Bold marks the minimum time in each row.}
  \label{tab:convergence}
  \small
  \begin{tabular}{lcrrr}
  \toprule
  Model & Loss & Megatron-LM & SYMI & Ours \\
  \midrule
  Qwen3-30B-A3B & 0.05 & 50.1 / 6{,}359 & 48.7 / 6{,}359 & \textbf{28.0} / 6{,}630 \\
                 & 0.02 & 66.4 / 8{,}440 & 64.6 / 8{,}440 & \textbf{36.7} / 8{,}693 \\
  \addlinespace
  \shortstack[l]{GLM4.5-\\106B-A12B}    & 2.66 & 22.8 / 2{,}898 & 22.2 / 2{,}898 & \textbf{17.5} / 4{,}151 \\
                 & 2.63 & 42.4 / 5{,}388 & 41.3 / 5{,}388 & \textbf{22.8} / 5{,}402 \\
  \bottomrule

  \end{tabular}
\end{table}

To relate the convergence behavior above to training efficiency, we compare measured cumulative MoE-layer training time at common target losses rather than at a fixed training step. This metric accumulates the time spent in complete MoE-layer execution, including forward and backward computation and communication, up to the first step at which each target loss is reached. It therefore measures the full MoE-layer training path rather than communication alone, but does not include execution outside the MoE layers or represent whole-model end-to-end training time. Taking the uncompressed trajectory as a reference, we use two target cross-entropy losses for each model: 0.05 and 0.02 for Qwen3-30B-A3B, and 2.66 and 2.63 for GLM4.5-106B-A12B. These thresholds sample two optimization levels after the initial rapid loss decrease. Both the uncompressed and C30 trajectories reach these targets within the observed training horizon. The thresholds provide common comparison points selected from Figure~\ref{fig:ppl 10k steps} for convergence. 

We primarily benchmark long-trajectory convergence against SYMI, as it represents the closest state-of-the-art framework decoupling expert training states. UltraEP primarily targets rack-scale communication primitives and hits restrictive memory boundaries under our sustained long-sequence training setup. Table~\ref{tab:convergence} reports the cumulative MoE-layer training time and the step at which each target is first reached. \method achieves the lowest cumulative time at all four targets. For Qwen3-30B-A3B, it requires 28.0 and 36.7 minutes, compared with 48.7 and 64.6 minutes for SYMI. For GLM4.5-106B-A12B, the corresponding times are 17.5 and 22.8 minutes, compared with 22.2 and 41.3 minutes. These results represent 21.2 to 44.8\% lower cumulative MoE-layer training time relative to SYMI, despite \method requiring more training steps to reach each target. The comparison thus accounts for both MoE-layer execution cost and the additional optimization steps associated with recalibration. While SYMI uses current-step expert loads to select placements, \method actively reduces the backward token to expert workload and balances the retained pairs across eligible expert instances. The measured reductions show that these execution savings substantially outweigh the additional steps at the evaluated loss targets.

\subsection{Throughput Analysis}
\begin{figure}
    \centering
    \includegraphics[width=\linewidth]
    {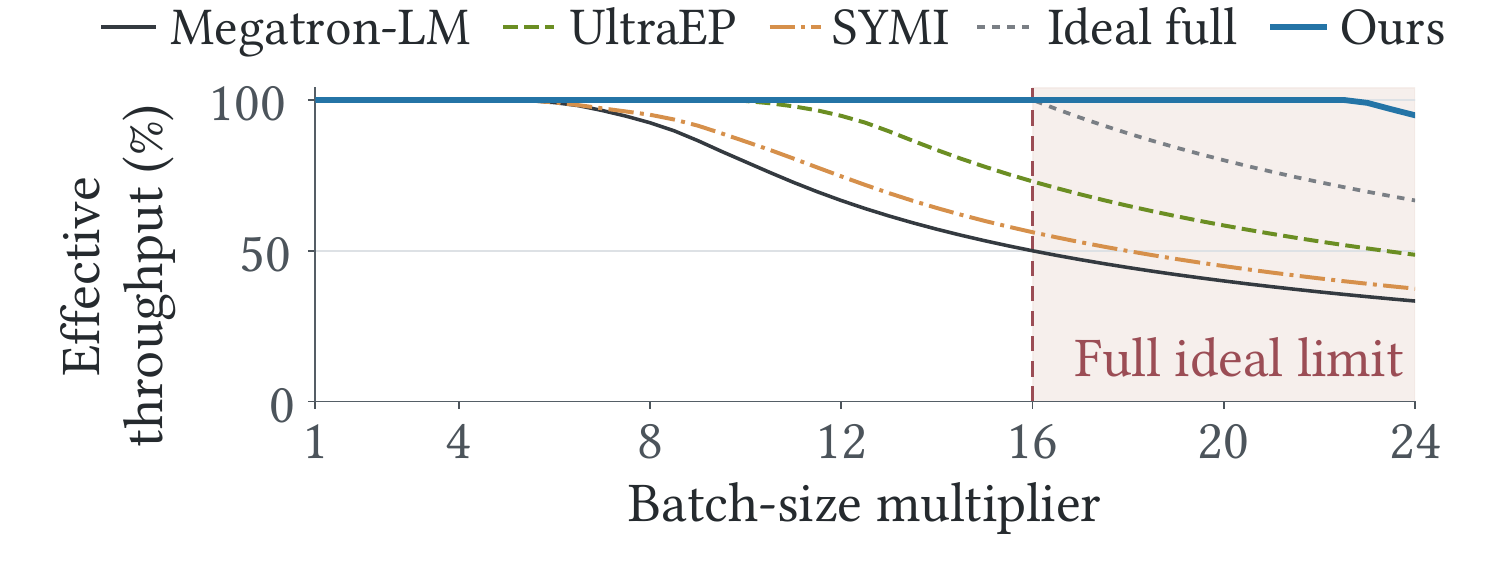}
    \caption{Effective backward throughput under increasing batch-size pressure (the base batch-size is 512 per-GPU). The vertical line marks the point at which even a perfectly balanced, uncompressed backward workload exceeds the modeled memory budget.}
    \label{fig:effective-throughput-vs-batch}
    \Description{Effective backward throughput versus batch-size pressure, with a marked modeled memory-capacity boundary.}
\end{figure}

Figure~\ref{fig:effective-throughput-vs-batch} evaluates how memory pressure affects effective throughput as batch size increases. We define effective throughput as the fraction of the offered batch whose corresponding backward workload fits within the modeled per-GPU memory budget. The experiment assumes an 80\,GiB GPU with a 64\,GiB resident footprint, leaving 16\,GiB for token buffers and additional expert replicas. Megatron-LM quickly loses throughput because its fixed placement concentrates tokens on overloaded GPUs. UltraEP and SYMI improve this placement by utilizing additional replicas, but their overall benefit remains strictly constrained by residual load imbalance and the memory required to host those replicas. The \emph{Ideal full} curve illustrates an uncompressed backward workload with perfectly balanced token traffic, optimistically charging zero memory for additional replicas. It reaches the memory limit at a batch-size multiplier of $16\times$. Beyond this threshold, placement strategies alone cannot accommodate the full backward workload within the given memory model. In contrast, our method first systematically reduces the backward token working set and then dynamically balances the retained pairs, requiring only 7 to 9 additional replicas across eight GPUs. Consequently, it preserves a 100\% effective throughput even at $20\times$, compared with 58.4\% for UltraEP, 44.9\% for SYMI, 40.0\% for Megatron-LM, and 80.0\% for the optimistic full-workload bound. This result directly characterizes the backward-capacity benefit of contribution-aware selection under the declared memory model.

\subsection{GPU Memory Breakdown}
\begin{figure}
    \centering

    \includegraphics[width=\linewidth]{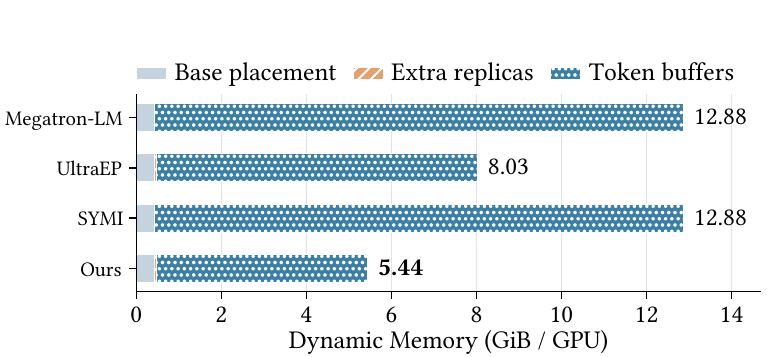}
    \caption{GPU-memory breakdown. Bars show the breakdown of each MoE layer.}
    \label{fig:moe-memory-analysis}
    \Description{Stacked bars compare memory used by base expert placement, extra replicas, and token buffers across methods.}
\end{figure}

Figure 7 separates base expert placement, additional replica storage, and token buffers. Their measured subtotal is 12.88 GiB for Megatron-LM, 8.03 GiB for UltraEP, 12.88 GiB for SYMI, and 5.44 GiB for \method. Forward load balancing distributes the complete routed workload across admitted expert instances, while backward recalibration and rerouting reduce and redistribute the retained workload. Together with compact execution, these mechanisms reduce the reported token-buffer component from 12.46 GiB under fixed ownership and 7.55 GiB under UltraEP to 4.97 GiB under \method. The three-component subtotal is therefore 57.8\% lower than Megatron-LM and 32.2\% lower than UltraEP. These reductions concern the displayed GPU allocations, not total device memory or host-memory optimizer state. These results demonstrate the concrete runtime advantages of our design. By combining contribution-aware selection with physical compaction, \method eliminates unselected rows prior to communication, directly trimming All-to-All transfer volumes. Furthermore, our variable-size grouped GEMMs execute compact per-instance workloads without incurring padding overhead. This substantial reduction along the critical GPU execution path forms the direct mechanical foundation for the overall training speedups reported in our end-to-end evaluations.

\subsection{Latency Breakdown}

\begin{figure}
    \centering

    \includegraphics[width=\linewidth]{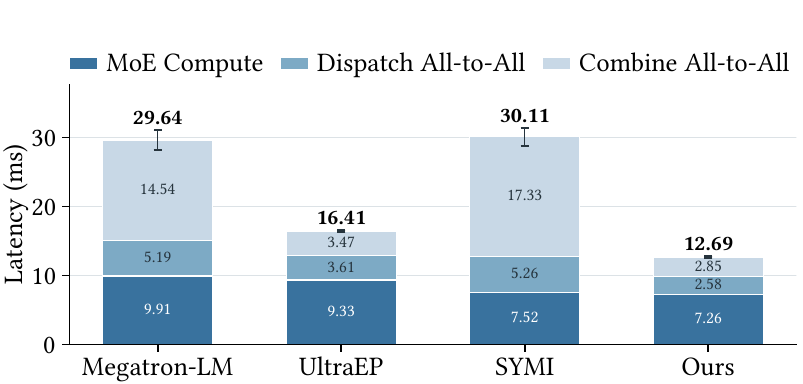}
    \caption{Measured latency per MoE-layer. Bars show MoE compute, dispatch All-to-All, and combine All-to-All latency breakdown.}
    \label{fig:latency-breakdown}
    \Description{Stacked bars compare MoE computation, dispatch All-to-All, and combine All-to-All latency across methods.}
\end{figure}

Using the same workload and resource constraints, Figure~\ref{fig:latency-breakdown} examines the computation and token-communication costs associated with these allocations. Dispatch transfers incoming gradients to expert executors, whereas combine returns input-gradient contributions to token owners. For each iteration, all three components are measured from the critical rank with the longest complete backward execution to preserve a consistent critical-path evaluation. Their mean subtotal is 29.64,ms for Megatron-LM, 16.41,ms for UltraEP, 30.11,ms for SYMI, and 12.69,ms for \method. Specifically, \method achieves an expert computation latency of 7.26,ms, a dispatch latency of 2.58,ms, and a combine latency of 2.85,ms.

These results demonstrate the concrete runtime advantages of our design. By combining contribution-aware selection with physical compaction, \method eliminates unselected rows prior to communication, directly trimming All-to-All transfer volumes. Furthermore, our variable-size grouped GEMMs execute compact per-instance workloads without incurring padding overhead. Consequently, \method reduces the backward execution subtotal by 57.2\% relative to Megatron-LM, 57.9\% relative to SYMI, and 22.7\% relative to UltraEP. This substantial reduction along the critical GPU execution path forms the direct mechanical foundation for the overall training speedups reported in Table~\ref{tab:convergence}.

\subsection{Ablation Study}
\begin{table}[t]
  \centering
  \caption{Mean $\Delta$PPL over 10000 steps.}
  \label{tab:mean-delta-ppl}
  \begin{tabular*}{\linewidth}{@{\extracolsep{\fill}}lrr@{}}
    \toprule
    Metric & Qwen3-30B-A3B & \shortstack{GLM4.5-\\106B-A12B} \\
    \midrule
    $\alpha_u^2
\lVert g_t\rVert_2^2
\lVert\phi_u\rVert_2^2$ (Ours)   & \textbf{0.012474} & \textbf{0.003464} \\
    L-2 Norm for Grad $\lVert g\rVert_2^2$   & 0.097654 & 0.011987 \\
    Absolute Value for Grad $|g|$                 & 0.107485 & 0.016662 \\
    Random Drop               & 0.775797 & 0.079282 \\
    \bottomrule
  \end{tabular*}
\end{table}

Table~\ref{tab:mean-delta-ppl} compares our LRM against alternative backward selection criteria over a sustained 10,000-step training trajectory, contrasting it with a squared gradient norm approach, an absolute gradient magnitude baseline, and a pure random selection mechanism. Our recalibration metric achieves the lowest mean $\Delta$PPL across both evaluated architectures, registering an average deviation of just 0.012474 for Qwen3-30B-A3B and an exceptionally tight 0.003464 for GLM4.5-106B-A12B. Both gradient-only metrics outperform random selection but produce larger PPL deviations than LRM. The ordering is consistent across both models: the squared gradient norm is the stronger gradient-only baseline, followed by the absolute gradient magnitude, while random selection yields the largest deviations. This empirical result strongly supports our core design hypothesis: properly selecting retained token to expert pairs requires considering the established routing strength and the saved expert-state magnitude alongside the raw backward gradient signal.

In distributed Expert Parallelism, two routed pairs can receive nearly identical upstream gradient norms but possess entirely different initial routing weights or saved activation magnitudes. A gradient-only rule assigns them identical importance, even when their actual mathematical contributions to an expert weight projection differ drastically. Our LRM accurately and efficiently accounts for these diverse factors. Crucially, it achieves this comprehensive scoring without requiring the computationally prohibitive materialization of full per-pair weight-gradient matrices. The tight PPL comparison provides strong empirical validation that this efficient rank-one local approximation preserves global model convergence and training quality over extensive training horizons.

\section{Related Work}
\label{sec:related-work}

\textbf{MoE routing and load balancing.} BASE layers balance token assignment, while Expert Choice fixes expert buckets and varies the number of selected experts per token~\cite{baselayers,expertchoice}. StableMoE distills and freezes a router to stabilize assignments~\cite{stablemoe}. Gating Dropout bypasses global routing for some tokens to reduce communication~\cite{gatingdropout}. These methods change forward routing. FasterMoE shadows popular experts and can alter expert selection to avoid network congestion~\cite{fastermoe}. FlexMoE adapts placement, while SmartMoE searches hybrid parallelization plans~\cite{flexmoe,smartmoe}. SYMI separates parameters from optimizer states, and LAER-MoE reconstructs expert instances from sharded parameters~\cite{symi,laermoe}. UltraEP and topology-aware TopoEP respond to realized post-gating loads~\cite{ultraep,topoep}. \method instead selects backward pairs under quotas and balances them across instances of their original experts.

\textbf{Selective backward computation.} Selective-Backprop prioritizes high-loss training examples~\cite{selectivebackprop}. In contrast, \method ranks routed pairs within each expert's quota. Its projection-level score uses the factorization underlying efficient per-example gradient-norm computation~\cite{goodfellowperexample}, applied to saved projection inputs and weighted incoming gradients.

\textbf{Memory and communication reduction.} ZeRO partitions training states to remove redundancy~\cite{zero}. ZeRO-Offload and ZeRO-Infinity extend memory management across GPU, CPU, and NVMe~\cite{zerooffload,zeroinfinity}. ES-MoE offloads expert parameters and pipelines execution~\cite{esmoe}. FlashAttention reduces attention HBM accesses through IO-aware tiling~\cite{flashattention}. These methods do not select expert-update contributions. RCMoE quantizes all-to-all exchanges and sparsifies all-reduce gradients~\cite{rcmoe}. Our row selection instead reduces token-dependent communication and expert computation without changing parameter dimensions.

\textbf{Parallel execution and overlap.} PipeDream pipelines minibatches across partitioned layers, while Alpa automatically plans parallel strategies~\cite{pipedream,alpa}. Within MoE layers, MegaBlocks uses block-sparse computation for irregular loads without capacity-based token dropping~\cite{megablocks}. Lancet, COMET, and FSMoE co-schedule and overlap communication with computation across different granularities~\cite{lancet,comet,fsmoe}. \method changes which backward rows enter execution and packs retained rows into contiguous expert segments. Unlike masking full-sized tensors, this compaction reduces physical GEMM row counts.

\section{Conclusion}

In this work, we present \method to tackle load imbalances and memory bottlenecks in Expert Parallelism. Unlike balancers that merely redistribute a fixed workload, \method tightly coordinates contribution-aware load recalibration with capacity-aware routing under strict memory and bandwidth constraints. It selectively retains high-value token-expert pairs to fit within planned quotas and routes them across admitted expert instances. Compact grouped GEMMs then translate this reduced workload into smaller communication payloads and lower computation costs without zero-padding overhead. By keeping the forward pass intact for accurate contribution scoring, \method advances MoE balancing by jointly controlling the size and spatial distribution of backward computation.

\bibliographystyle{ACM-Reference-Format}
\bibliography{references}
\appendixpage
\section{Derivation and Optimality of Token--Expert Recalibration}
\label{app:metric-proof}

This appendix proves Proposition~\ref{prop:Calibration-metric} and the
optimality of the resulting drop rule under the fidelity objective used by
our recalibration controller.  The first result is an exact, projection-level
identity for one routed pair.  The second result concerns selection: for a
fixed number of equal-size rows to remove, dropping the rows with the smallest
scores minimizes the sum of their isolated projection-gradient
perturbations.  We state the assumptions explicitly because this selection
result is not a claim about the nonlinear loss or the entire optimizer
trajectory.

\subsection{Pairwise projection-gradient contribution}
\label{app:metric-projection}

Consider one completed forward route and the protected output projection
(FC2) of expert $e$.  Let
\begin{equation}
  W_e\in\mathbb R^{d_{\mathrm{out}}\times d_{\mathrm{in}}},
  \qquad
  \phi_u\in\mathbb R^{d_{\mathrm{in}}},
  \qquad
  g_t\in\mathbb R^{d_{\mathrm{out}}},
\end{equation}
where $u=(t,e)$ is a routed token--expert pair.  The saved vector $\phi_u$ is
the actual input to the protected projection.  We take $g_t$ to be the
upstream cotangent before multiplication by the fixed routing weight
$\alpha_u$.  Equivalently, an implementation may expose the already weighted
cotangent
\begin{equation}
  \delta_u:=\alpha_u g_t,
  \label{eq:appendix-routed-cotangent}
\end{equation}
in that representation, the routing weight must not be applied again.

The contribution of pair $u$ to the routed expert output is
\begin{equation}
  y_u=\alpha_u W_e\phi_u.
  \label{eq:appendix-pair-output}
\end{equation}
The forward route, $\alpha_u$, and $\phi_u$ are fixed during recalibration.
For an arbitrary perturbation $dW_e$, the differential of the loss through
this pair is
\begin{align}
  d\mathcal L_u
  &=g_t^\top d y_u
    =g_t^\top\alpha_u(dW_e)\phi_u
    \nonumber\\
  &=\operatorname{tr}\!\left(
      \bigl((\alpha_u g_t)\phi_u^\top\bigr)^\top dW_e
    \right).
  \label{eq:appendix-pair-differential}
\end{align}
By the defining identity
$d\mathcal L=\operatorname{tr}((\nabla_W\mathcal L)^\top dW)$, the isolated
contribution of pair $u$ to the protected weight gradient is therefore
\begin{equation}
  G_u
  :=\left[\nabla_{W_e}\mathcal L\right]_u
  =(\alpha_u g_t)\phi_u^\top
  =\delta_u\phi_u^\top.
  \label{eq:appendix-pair-gradient}
\end{equation}

Let $\mathcal U_e$ denote all pairs contributing to the same protected
projection.  Before omission, its routed-pair gradient is
\begin{equation}
  G_e=\sum_{v\in\mathcal U_e}G_v.
\end{equation}
If only pair $u$ is omitted and reconstructed as zero while every other pair
is unchanged, then
\begin{equation}
  G_e^{(-u)}=G_e-G_u,
  \qquad
  \Delta_u:=G_e^{(-u)}-G_e=-G_u.
  \label{eq:appendix-single-omission}
\end{equation}
The sign is immaterial after taking a norm, but Eq.~\eqref{eq:appendix-single-omission}
makes clear that the metric measures a backward contribution removed from a
fixed completed route, it does not rerun or alter the forward computation.

For vectors $a$ and $b$, the squared Frobenius norm of their outer product
factorizes as
\begin{align}
  \lVert ab^\top\rVert_F^2
  &=\operatorname{tr}\!\left((ab^\top)^\top(ab^\top)\right)
    \nonumber\\
  &=\operatorname{tr}\!\left(b(a^\top a)b^\top\right)
    =(a^\top a)(b^\top b)
    =\lVert a\rVert_2^2\lVert b\rVert_2^2.
  \label{eq:appendix-rank-one-identity}
\end{align}
Applying Eq.~\eqref{eq:appendix-rank-one-identity} with
$a=\alpha_u g_t$ and $b=\phi_u$ proves
\begin{align}
  \lVert\Delta_u\rVert_F^2
  &=\left\lVert(\alpha_u g_t)\phi_u^\top\right\rVert_F^2
    \nonumber\\
  &=\alpha_u^2\lVert g_t\rVert_2^2\lVert\phi_u\rVert_2^2
    =R_u,
  \label{eq:appendix-calibration-metric}
\end{align}
which is exactly Eq.~\eqref{eq:Calibration-metric}.  Hence the three factors
in $R_u$ are not an ad hoc weighted sum: together they form the exact squared
norm of the isolated gradient contribution to the protected projection.

\subsection{Optimality of the drop rule}
\label{app:metric-selection}

Let $\mathcal U$ be the causally available token--expert pairs at the
recalibration boundary, with $N=|\mathcal U|$.  Let
$D\subseteq\mathcal U$ be the dropped set and
$S=\mathcal U\setminus D$ the retained set.  All rows at this boundary have
the same shape and encoding, and therefore the same communication cost.  For
a prescribed drop count $d$, or equivalently a retained count $K=N-d$, define
the additive local calibration distortion
\begin{equation}
  \widehat{\mathcal E}(D)
  :=\sum_{u\in D}\lVert\Delta_u\rVert_F^2
   =\sum_{u\in D}R_u.
  \label{eq:appendix-additive-drop-objective}
\end{equation}

\begin{proposition}[Optimal score-based dropping]
\label{prop:appendix-optimal-drop}
Under a fixed completed route, equal row costs, and a
fixed drop count $d$, any set containing the $d$ smallest scores $R_u$
minimizes Eq.~\eqref{eq:appendix-additive-drop-objective}.  Equivalently, any
retained set containing the $K=N-d$ largest scores maximizes preserved
calibration value.  The solution is unique when there is a strict gap between
the $d$th and $(d+1)$st scores in ascending order, otherwise every difference
between optimal solutions is confined to tied boundary scores.
\end{proposition}

\begin{proof}
Let $D$ be any dropped set with $|D|=d$.  Suppose that $D$ contains a pair
$i$ while retaining a pair $j$ such that $R_i>R_j$.  Replacing $i$ by $j$
gives
\begin{equation}
  D'=(D\setminus\{i\})\cup\{j\},
\end{equation}
and changes the objective by
\begin{equation}
  \widehat{\mathcal E}(D')-
  \widehat{\mathcal E}(D)
  =R_j-R_i<0.
  \label{eq:appendix-drop-exchange}
\end{equation}
Thus a set with such an inversion cannot be optimal.  Repeatedly applying
this exchange terminates at a set containing the $d$ smallest scores, because
each exchange strictly decreases a finite-set objective.  Conversely, a set
containing the $d$ smallest scores has no improving exchange: every retained
score is at least as large as every dropped score.
It therefore minimizes the objective.  A strict boundary gap fixes the set,
ties yield multiple sets with the same sum.
\end{proof}

The equivalent retained-set form follows directly from
\begin{equation}
  \sum_{u\in D}R_u
  =\sum_{u\in\mathcal U}R_u-\sum_{u\in S}R_u.
  \label{eq:appendix-drop-retain-complement}
\end{equation}
The first term is constant, so minimizing dropped value is identical to
maximizing retained value.  This is the precise sense in which our drop is
optimal: once the load controller determines that $d$ equal-size rows must be
removed, no other $d$-row mask preserves more of the declared pairwise
calibration value.

When the controller imposes a fixed drop count $d_e$ separately on each
expert or overloaded path, the same proof applies independently inside every
partition: dropping the $d_e$ smallest scores in each partition minimizes the
sum of Eq.~\eqref{eq:appendix-additive-drop-objective} over all masks obeying
those cardinalities.  More general placement and capacity constraints define
a restricted feasible family $\mathcal F$.  The corresponding selection
problem is
\begin{equation}
  D^*\in\arg\min_{D\in\mathcal F}
  \sum_{u\in D}R_u.
  \label{eq:appendix-constrained-drop}
\end{equation}
For the fixed per-expert quotas used by our controller, this objective
separates across experts and is minimized by selecting the smallest scores
within each expert.

\subsection{Relation to the aggregate expert gradient}
\label{app:metric-aggregate}

The single-pair identity in
Eq.~\eqref{eq:appendix-calibration-metric} is exact.  For several dropped
pairs from the same expert, their contributions to a shared $W_e$ add.
Here let $D\subseteq\mathcal U_e$ and $d=|D|$.  The actual change in the
protected expert gradient is
\begin{equation}
  \Delta_D=-\sum_{u\in D}G_u,
\end{equation}
and therefore
\begin{align}
  \lVert\Delta_D\rVert_F^2
  &=\sum_{u\in D}R_u
    +2\sum_{\substack{u<v\\u,v\in D}}
      \langle G_u,G_v\rangle_F.
  \label{eq:appendix-aggregate-cross-terms}
\end{align}
The cross terms can be positive or negative.  Consequently,
Proposition~\ref{prop:appendix-optimal-drop} exactly optimizes the additive
sum of isolated perturbations, it does not, without further assumptions,
claim to minimize the realized norm of the summed shared-parameter gradient
error.

Two useful special cases recover an aggregate interpretation.  First, if the
dropped contributions are pairwise Frobenius-orthogonal, all cross terms in
Eq.~\eqref{eq:appendix-aggregate-cross-terms} vanish and
$\lVert\Delta_D\rVert_F^2=\widehat{\mathcal E}(D)$.  Second, if the cross terms of the selected dropped set have zero
expectation, the same equality holds in expectation.  Without either condition, the
triangle and Cauchy--Schwarz inequalities still give the deterministic bound
\begin{equation}
  \lVert\Delta_D\rVert_F^2
  \le
  \left(\sum_{u\in D}\sqrt{R_u}\right)^2
  \le d\sum_{u\in D}R_u.
  \label{eq:appendix-aggregate-bound}
\end{equation}
For fixed $d$, the smallest-score rule also minimizes the final upper bound
in Eq.~\eqref{eq:appendix-aggregate-bound}.

\subsection{Scope of the result}
\label{app:metric-scope}

The proof relies on four declared conditions.  First, the forward route and
saved projection inputs are fixed, recalibration acts only on the completed
route's backward contributions.  Second, $\phi_u$ is the actual input to the
protected linear projection.  For a nonlinear multi-layer expert, the same
formula is exact for the selected projection but is not the squared norm of
the expert's complete parameter-gradient contribution.  Third, the closed-form
drop rule assumes equal row costs and a fixed cardinality.  Unequal byte costs
lead to a knapsack problem, while replica, topology, and HBM constraints lead
to Eq.~\eqref{eq:appendix-constrained-drop}.  Fourth, the objective protects a
local projection-gradient quantity.  It is neither an identity for nonlinear
loss change nor a guarantee about future optimizer states or model quality,
those effects are evaluated empirically.

Within these conditions, the result is exact: $R_u$ is the squared isolated
perturbation caused by dropping pair $u$, and our drop rule is optimal for the
declared additive calibration objective at the required recalibration rate.

\end{document}